\documentclass[sigconf, natbib=true]{acmart}

\AtBeginDocument{%
  \providecommand\BibTeX{{%
    \normalfont B\kern-0.5em{\scshape i\kern-0.25em b}\kern-0.8em\TeX}}}

\copyrightyear{2022}
\acmYear{2022}
\setcopyright{acmcopyright}\acmConference[WWW '22]{Proceedings of the ACM Web Conference 2022}{April 25--29, 2022}{Virtual Event, Lyon, France}
\acmBooktitle{Proceedings of the ACM Web Conference 2022 (WWW '22), April 25--29, 2022, Virtual Event, Lyon, France}
\acmPrice{15.00}
\acmDOI{10.1145/3485447.3512100}
\acmISBN{978-1-4503-9096-5/22/04}

\usepackage{multirow}
\usepackage{color}
\usepackage{graphics}

\usepackage{xcolor, soul}
\sethlcolor{lightgray}

\usepackage{tabularx}

\definecolor{aluminum}{RGB}{153,153,153}
\definecolor{platinum}{RGB}{228,228,228}
\definecolor{bgc}{RGB}{245,245,245}
\definecolor{gallery}{RGB}{240,240,240}
\definecolor{tuatara}{RGB}{67, 67, 67}
\definecolor{flamingo}{RGB}{237, 88, 85}
\definecolor{salmon}{RGB}{242,131,107}
\definecolor{free_speech_aquamarine}{RGB}{0, 156, 114}

\definecolor{Aquamarine3}{RGB}{102, 205, 170} 
\definecolor{Goldenrod1}{RGB}{255, 193, 37} 
\definecolor{IndianRed1}{RGB}{255, 106, 106} 
\definecolor{SlateBlue1}{RGB}{131, 111, 255}

\definecolor{bb}{HTML}{95e1d3}
\definecolor{gg}{HTML}{c7ffd8}
\definecolor{yy}{HTML}{f0c38e}
\definecolor{blu}{HTML}{5ab4ba}
\definecolor{rr}{HTML}{f38181}

\definecolor{c1}{HTML}{6E85B2}
\definecolor{c2}{HTML}{368B85}
\definecolor{c3}{HTML}{C56824}
\definecolor{c4}{HTML}{FFC069}
\definecolor{c5}{HTML}{916BBF}

\usepackage{hyperref}
\usepackage{subcaption}

\begin{document}

\title{Meta-Weight Graph Neural Network: Push the Limits Beyond Global Homophily}

\author{Xiaojun Ma}
\email{mxj@pku.edu.cn}
\affiliation{
\institution{
Key Lab. of Machine Perception (MoE), School of AI, Peking University
}
\city{Beijing}
\country{China}
}

\author{Qin Chen}
\email{
     chenqink@pku.edu.cn
}
\affiliation{
\institution{ 
Key Lab. of Machine Perception (MoE), School of AI,  Peking University
}
\city{Beijing}
\country{China}
}

\author{Yuanyi Ren}
\email{
    celina@pku.edu.cn
}
\affiliation{
\institution{
Key Lab. of Machine Perception (MoE), School of AI, Peking University
}
\city{Beijing}
\country{China}
}

\author{Guojie Song}
\authornote{Corresponding author.}
\email{
    gjsong@pku.edu.cn
}
\affiliation{
\institution{
Key Lab. of Machine Perception (MoE), School of AI, Peking University
}
\city{Beijing}
\country{China}
}

\author{Liang Wang}
\email{
    liangbo.wl@alibaba-inc.com
}
\affiliation{
\institution{Alibaba Inc}
\city{Beijing}
\country{China}
}







\begin{abstract}
Graph Neural Networks (GNNs) show strong expressive power on graph data mining, by aggregating information from neighbors and using the integrated representation in the downstream tasks. The same aggregation methods and parameters for each node in a graph are used to enable the GNNs to utilize the homophily relational data. However, not all graphs are homophilic, even in the same graph, the distributions may vary significantly. Using the same convolution over all nodes may lead to the ignorance of various graph patterns. Furthermore, many existing GNNs integrate node features and structure identically, which ignores the distributions of nodes and further limits the expressive power of GNNs.  To solve these problems, we propose Meta Weight Graph Neural Network (MWGNN) to adaptively construct graph convolution layers for different nodes. First, we model the Node Local Distribution (NLD) from node feature, topological structure and positional identity aspects with the Meta-Weight. Then, based on the Meta-Weight, we generate the adaptive graph convolutions to perform a node-specific weighted aggregation and boost the node representations. Finally, we design extensive experiments on real-world and synthetic benchmarks to evaluate the effectiveness of MWGNN. These experiments show the excellent expressive power of MWGNN in dealing with graph data with various distributions. 
\end{abstract}




\maketitle

\section{Introduction}
\label{sec:introduction}

As a powerful approach to extracting and learning information from relational data, Graph Neural Networks (GNNs) have flourished in many applications, including molecules \cite{gcpn, LanczosNet}, social networks \cite{GovernCascade}, biological interactions \cite{daggnn}, and more. 
Among various techniques \cite{GNNScarselli09, bruna-gcn, defferrard2016convolutional, kipf-gcn}, Graph Convolutional Network (GCN) stands out as a powerful and efficient model. Since then, more variants of GNNs, such as GAT \cite{gat}, SGCN \cite{sgc}, GraphSAGE \cite{sage}, have been proposed to learn more powerful representations. These methods embrace the assumption of homophily on graphs, which assumes that connected nodes tend to have the same labels.
Under such an assumption, the propagation and aggregation of information within graph neighborhoods are efficient in graph data mining. 

Recently, some research papers \cite{H2GCN,AMGCN,GPRGNN,CPGNN,Geom-GCN} propose to adapt the Graph Convolution to extend GNNs' expressive power beyond the limit of homophily assumption, because there do exist real-world graphs with heterophily settings, where linked nodes are more likely to have different labels. 
These methods improve the expressive power of GNNs by changing the definition of neighborhood in graph convolutions. For example, H2GCN \cite{H2GCN} extend the neighborhood to higher-order neighbors and AM-GCN \cite{AMGCN} constructs another neighborhood considering feature similarity. Changing the definition of neighborhood do help GNNs to capture the heterophily in graphs. 

\begin{figure} \centering
  \begin{subfigure}[b]{0.48\linewidth}
      \includegraphics[width=.99\linewidth]{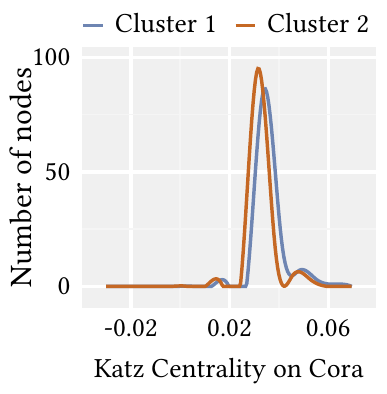}
  \end{subfigure} %
  \begin{subfigure}[b]{0.48\linewidth}    
      \includegraphics[width=.99\linewidth]{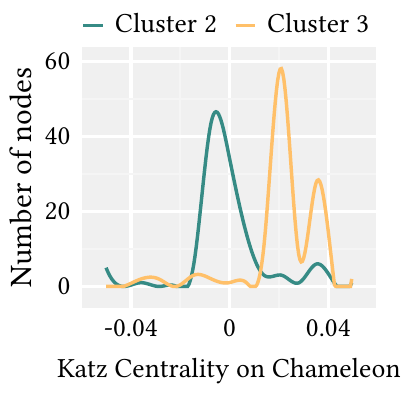}  
  \end{subfigure} 
  \caption{Katz centrality distribution of clustered Cora (in class 1) and Chameleon (in class 4). The nodes are clustered and then the Katz Centrality distribution is plotted for nodes of the same label but belonging in two different clusters. }
  \label{fig:katz}
\end{figure}


However, homophily and heterophily is only a simple measurement of the NLD, because they only consider the labels of the graphs. In real-world graphs, topological structure, node feature and positional identity play more critical role than labels, especially for graphs with few labels. So the models mentioned above lack the ability to model NLD.
We visualize the Katz centrality \cite{katz1953new}, a measurement of the topological structure of nodes, in \autoref{fig:katz} to show two situations. We first separate the graph with METIS \cite{karypis1998fast} and show nodes of one label in Cora and Chameleon. In Cora, the Katz centrality reveals high consistency in different parts, while that of Chameleon shows an obvious distinction between different parts.

In addition to the difference in topological structure distributions, there are also variances in node feature distributions. Besides, the correlation between topological structure and node feature distributions is not consistent. 
Therefore, using one graph convolution to integrate the topological structure and node feature information leads to ignorance of such complex correlation. 

In conclusion, there are two challenges limiting the expressive power of GNNs: (1) the complexity of Node Local Distributions (2) the inconsistency correlation between node feature and topological structure. To overcome these challenges, we propose \textit{Meta-Weight Graph Neural Network}. 
For the first challenge, We model the NLD in topological structure, node feature, positional identity fields with Meta-Weight. In detail, the three types of NLD is captured using Gated Recurrent Unit (GRU) and Multi-Layer Perception (MLP) separately and then combined using an attention mechanism.
For the second challenge, based on Meta-Weight, we adaptively conduct graph convolution with two aggregation weights and three channels. The proposed two aggregation weights decouple the correlation between node feature and topological structure. To further model the complex correlation, we boost the node representations with the entangled channel, node feature channel and topological structure channel, respectively. 
We conduct experiments on semi-supervised node classification. The excellent performance of MWGNN is demonstrated empirically on both real-world and synthetic datasets. 
Our major contributions are summarized as follows:
\begin{itemize}

\item 
We demonstrate the insufficient modeling of NLD of existing GNNs and propose a novel architecture MWGNN which successfully adapt the convolution learning for different distributions considering topological structure, node feature, and positional identity over one Graph. 


\item 
We propose the Meta-Weight mechanism to describe the complicated NLD and 
the adaptive convolution based on Meta-Weight to boost node embeddings with decoupled aggregation weights and independent convolution channels. 

\item 
We conduct experiments on real-world and synthetic datasets to demonstrate the superior performance of MWGNN.
Especially, MWGNN gains an accuracy improvement of over 20\% on graphs with the complex NLD. 
\end{itemize}

\section{Preliminary}

Let $\mathcal G = (\mathcal V, \mathcal E)$ be an undirected, unweighted graph with node set $\mathcal V$ and edge set $\mathcal E$. 
Let $ | \mathcal V | = N $. 
We use $\boldsymbol{A} \in \{0,1 \}^{ N \times N }$ for the adjacency matrix,  $\boldsymbol{X} \in \mathbb{R}^{ N \times F }$ for the node feature matrix, 
and $\boldsymbol y \in \mathbb{R}^{N}$ for the node label matrix.
Let $\mathcal N_i$ denote the neighborhood surrounding node $v_i$, and $\mathcal{N}_{i,k} = \{ v_j | d(v_i,v_j) \leq k \}$ denote $v_i$'s neighbors within $k$ hops.


\subsection{Graph Neural Networks}
\label{sec:gnn}

Most Graph Neural Networks formulates their propagation mechanisms by two phases, the aggregation phase and the transformation phase.
The propagation procedure can be summarized as
\begin{equation}
    \boldsymbol{H}_{i}^{(l+1)} = \operatorname{TRANS}  \left( \operatorname{AGG} \left( \boldsymbol{H}_{i}^{(l)} , \left\{\boldsymbol{H}_{j}^{(l)}: v_j \in N_i \right\} \right) \right),
\end{equation}
where $\boldsymbol H^{(l)} \in \mathbb{R}^{N \times d^{(l)}}$ stands for the embedding of the $k$-th layer and $\boldsymbol H^{(0)}= \boldsymbol X$, $d^{(l)}$ is the dimension of $l$-th layer representations.  
$\operatorname{AGG}(\cdot)$ denotes the function aggregating $\boldsymbol H^{(k)}$, and  $\operatorname{TRANS}(\cdot)$ is a layer-wise transformation function including a weight matrix $\boldsymbol W^{(l)}$ and the non-linear activation fuctions (e.g. ReLU).


\subsection{Global and Local Homophily}

Here we define global and local homophily ratio to estimate the homogeneity level of a graph.
\begin{definition}[Global Edge Homophily]

We define Global Edge Homophily ratio \cite{H2GCN}  $h$ as a measure of the graph homophily level:
\begin{equation}
    h = \frac{| \{(v_i, v_j) : (v_i, v_j) \in \mathcal E \wedge \boldsymbol y_i = \boldsymbol y_j \}|} {|\mathcal E|},
\end{equation}
$h$ represents the percentage of edges  connecting nodes of the same  label in the edge set $\mathcal E$,
Graphs with strong homophily may have a high global edge homophily ratio up to 1, while those with low homophily embrace a low global edge homophily ratio down to 0.
\end{definition} 

\begin{definition}[Local Edge Homophily]
For node $v_i$ in a graph, we define the Local Edge Homophily ratio $h_i$ as a measure of the local homophily level surrounding node $v_i$:


\begin{equation}
    h_i = \frac{|\{(v_i, v_j) : v_j \in \mathcal N_i \wedge \boldsymbol y_i = \boldsymbol y_j)\}|} {|\mathcal N_i|},
    \label{eq:local edge homophily}
\end{equation}
$h_i$ directly represents the edge homophily in the neighborhood $\mathcal N_i$ surrounding node $v_i$.  
\end{definition}

\section{Meta-Weight Graph Neural Network}
\label{sec:model}

\begin{figure*}[htb]
    \centering
    \includegraphics[width=0.98\textwidth]{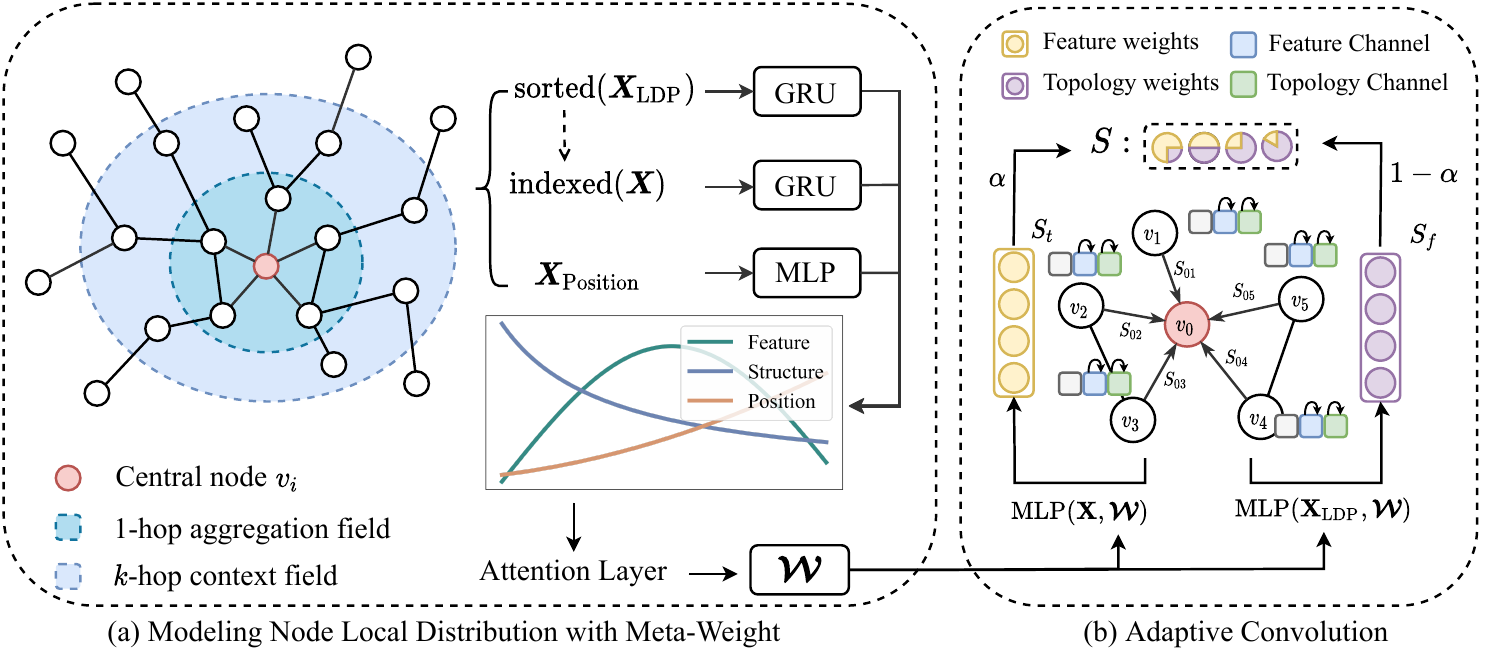} 
    \caption{The framework of MWGNN. (a) Generate the Meta-Weight considering $k$-hop context field for central nodes. First we learn three local distributions in topological structure, node feature, and positional identity fields and integrate them with an attention layer. (b) Based on the Meta-Weight, we propose the Adaptive Convolution. By generating $\boldsymbol{S}_t, \boldsymbol{S}_f$ and adaptively fusing them with a hyper-parameter $\alpha$, the Adaptive Convolution aggregates the neighbors. Then two additional Independent Convolution Channels are proposed to boost the node representations efficiently. }
    \label{fig:framework}
\end{figure*}

\paragraph{Overview}
In this section, we introduce the proposed method MWGNN. The MWGNN framework consists of two stages: (a) modeling Node Local Distribution by Meta-Weight and (b) adaptive convolution. The visualization of the framework is shown in \autoref{fig:framework}. 
First, we generate the Meta-Weight to model the NLD considering topological structure, node feature, and positional identity distributions separately. Then we integrate them via an attention mechanism, as shown in (a).
Next, the key contributions in (b) is the Adaptive Convolution consisting of the Decoupled Aggregation Weights and Independent Convolution Channels for node feature and topological structure.



\subsection{Modeling Node Local Distribution with Meta-Weight}
\label{sec:model-a}
In this stage, we aim to learn a specific \textit{key} to guide the graph convolution adaptively.
As discussed in \autoref{sec:introduction}, the complex Node Local Distribution hinders the expressive power of GNNs. For the node classification task, GNNs essentially map the combination of node features and topological structure from neighborhood to node labels. Using the same convolution over all the nodes and the pre-defined neighborhood, most existing GNNs can only project a fixed combination to node labels. Therefore, these GNNs achieve satisfactory results on graphs with simple NLD (e.g. homophily) while failing to generalize to graphs with complex NLD (e.g. heterophily). 

To push the limit on graphs with complex NLD and conduct adaptive convolution on nodes, first we need to answer the question: \textit{ What exactly NLD is and how to explicitly model it?} 
Node Local Distribution is a complex and general concept, in this paper, when discussing NLD, we refer to the node patterns in topological structure, node feature, and positional identity fields. 
Topological structure and node feature are widely used in the learning of node representations. 

However, only using topological structure and node feature limits the expressive power of GNNs, because some nodes can not be recognized in the computational graphs \cite{idgnn}.  
IDGNN proposes a solution to improve the expressive power of GNNs than the 1-Weisfeiler-Lehman (1-WL) test \cite{1wltest}, by inductively injecting nodes’ identities during message passing.
It empowers GNNs with the ability to count cycles, differentiate random $d$-regular graphs, etc. 
Nevertheless, the task of modeling NLD is much more complex. Thus we need a stronger identification.  
So we introduce the positional identity to model the NLD along with topological structure and node feature. 
In general, we learn three representation matrices $\boldsymbol{\mathcal D}_t, \boldsymbol{\mathcal D}_f, \boldsymbol{\mathcal D}_p \in \mathbb{R}^{N \times d_{\text{Meta}}}$ to formulate these three node patterns.

Finally, the three kinds of node patterns may contribute unequally to NLD for different nodes, so we adopt an attention mechanism to combine them automatically. 

\subsubsection{Node Patterns in Topological Structure, Node Feature, and Positional Identity Fields}
We elaborate on modeling the local node patterns in topological structure, node feature, and positional identity fields. 

\paragraph{Topological Structure Field}
Topology information is an essential part of graph data. Most GNNs implicitly capture topology information by aggregating the representations of neighbor nodes, and combining feature information with topology information. Other methods, Struct2Vec \cite{struc2vec} for example, explicitly learn the topology information and express it with vectors. Usually, they utilize computationally expensive mechanisms like Random Walk to explicitly extract the topology information. In this paper, we propose a simple but effective method using the degree distribution to capture node patterns in the topological structure field because the degree distribution can describe the topological structure \cite{Statistical_mechanics_of_complex_networks}.

For a node $v_i$, the sub-graph induced by $\mathcal{N}_{i,k}$ contains the related topology information. 
The degree distribution of $\mathcal{N}_{i,k}$ describes the local topological structure. To make this distribution graph-invariant, we sort the degrees by their value. 
It is noted that we use a statistical description, the Local Degree Profile (LDP) \cite{cai2018simple} to provide multi-angle information: 
\begin{equation*}
\begin{aligned}
    \boldsymbol{X}_{\text{LDP},i} = \big[d_i, &\operatorname{MIN} ( \operatorname{DN}_i ), \operatorname{MAX} (\operatorname{DN}_i), \\
    &\operatorname{MEAN} (\operatorname{DN}_i), \operatorname{STD}(\operatorname{DN}_i) \big],
\end{aligned}
\end{equation*}

where $\boldsymbol{X}_{\text{LDP},i}$ is the LDP for node $v_i$, $d_i$ is the degree of $v_i$ and $ \operatorname{DN}_i = \{d_j | v_j \in \mathcal N_i \}$. 

As a sequential data, the sorted LDP can be well studied by Gated Recurrent Unit \cite{gru} as:
\begin{equation}
\small
\begin{aligned} 
    \boldsymbol u^{(t)} &= \sigma \left( \boldsymbol W_u \left[ \boldsymbol h^{(t-1)}, \boldsymbol x^{(t)} \right] + \boldsymbol{b}_u \right) \\ 
    \boldsymbol r^{(t)} &= \sigma \left( \boldsymbol{W}_{r} \left[ \boldsymbol{h} ^{(t-1)}, \boldsymbol{x}^{(t)}\right] + \boldsymbol{b}_r  \right) \\
    \hat{\boldsymbol{h}}^{(t)} &= \tanh \left( \boldsymbol{W}_h \left[ \boldsymbol{r}^{(t)} \odot   \boldsymbol{h} ^{(t-1)}, \boldsymbol{x}^{(t)} \right] + \boldsymbol{b}_h  \right) \\
    \boldsymbol{h}^{(t)}&=\left(1-\boldsymbol{u}^{(t)}\right) \odot \boldsymbol{h}^{(t-1)}+\boldsymbol{u}^{(t)} \odot \hat{\boldsymbol{h}}^{(t)} \\
\end{aligned}
\label{eq:GRU}
\end{equation}
where $\boldsymbol{x}^{(t)}\in \mathbb{R}^{5}$ is the $t$-th LDP vector of the sorted degree sequence of $\mathcal{N}_{i,k}$, $\odot$ denotes element-wise multiplication, and $\boldsymbol{h}^{(t)}$ denotes the hidden states of $t$-th step of GRU. 
The output $\boldsymbol{\mathcal D}_t =  \boldsymbol{h}^{(T)}$ of GRU is the learnt topological structure distribution, where $T$ is the layer number of GRU.

\paragraph{Node Feature Field} 
Usually, patterns in node feature field are captured by directly applying neural networks on the input $\boldsymbol{X}$.
On the one hand, this explicit estimation of node feature patterns enables any off-the-shelf neural networks.
On the other hand, thanks to the flourish of neural networks, we can accommodate different distributions of node features. 

In learning the distributions of topological structure, we sort the degree sequence by the value of the degrees, which enables the estimator to perform the graph-agnostic calculation.
To maintain the node-level correspondence between feature and degree sequence, we sort the node feature $\boldsymbol{X}$ along the nodes' dimension to keep them in the same order.
For node $v_i$, the sorted node feature sequence is:
\begin{equation*} 
\small
    \boldsymbol{X}_{\text{Feature}} = \operatorname{SORT} \left(
     \{ \boldsymbol{X}_j | v_j \in \mathcal{N}_{i,k} \}
    \right),
\end{equation*}
where the sorting function $\operatorname{SORT} (\cdot)$ denotes the order by nodes' degrees. 
We consider two types of neural work to model Feature Pattern:
\begin{itemize}
\item GRU. 
The mechanism to learn the distribution of node feature pattern resembles the one of topological structure pattern.
Taking $\boldsymbol{X}_{\text{Feature}}$ as input, another GRU as \autoref{eq:GRU} is applied.
In this way, GRU reckons the distribution of node feature pattern, and the output $\boldsymbol{\mathcal D}_f$ of GRU is the learnt node feature distribution. 

\item Average.
The average operation can be viewed as the simplest neural networks without parameters.
This is a graph-agnostic method without the need to sort the node feature. 
Simply taking the average of $\{ \boldsymbol{X}_j | v_j \in \mathcal{N}_{i,k} \}$, we have the summary of the feature distribution. 
In the experiment, the most results are carried by MWGNN with the average operation, 
which illustrates that the specific implementation is not essential to the ability of meta-weight $\mathcal W$.
On the contrary, the general design of the local distribution generator exerts a more significant impact on the results.
\end{itemize}

\paragraph{Positional Identity Field} 
The position embedding is widely used in Transformer to enhance models' expressive power.
In fact, position embedding has been used in graphs in another form \cite{pgnn}, by sampling sets of anchor nodes and computing the distance of a given target node to each anchor-set. 
However, explicitly learning a non-linear distance-weighted aggregation scheme over the anchor-sets \cite{pgnn} is computationally intensive and requires much space for the storage of anchor nodes. 
Therefore, we use the distance of the shortest path (SPD) between any two nodes as the position embedding \cite{graphormer}, which helps the model accurately capture the spatial dependency in a graph. 
To be concrete, the position embedding of node $v_i$ is:
\begin{equation}
\small
    \boldsymbol{X}_{\text{Position},i} = \left( \phi (v_i,v_1), \phi (v_i,v_2), \cdots, \phi (v_i,v_n) \right),
\end{equation} 
here $\phi (v_i,v_j)$ denotes the SPD between nodes $v_i$ and $v_j$ if the two nodes are connected. If not, we set the output of $\phi$ to be a special value, i.e., -1. 
Then, a neural network $\Phi (\cdot)$ is applied on the position embedding $\boldsymbol{X}_{\text{Position}}$ to model the distributions of positional identity pattern. Finally, we have $\boldsymbol{\mathcal D}_p = \Phi \left( \boldsymbol{X}_{\text{Position}}  \right)$  as the positional identity distribution.

\subsubsection{ Integration of Three Distributions }
\label{sec:att}

The above-mentioned process models three specific local distributions: topological structure distribution, $\boldsymbol{\mathcal D}_t$,  node feature distribution $\boldsymbol{\mathcal D}_f$, and positional identity distribution $\boldsymbol{\mathcal D}_p$.
The overall local distribution of nodes could be correlated with one of them or their combinations to different extents. Thus we use an attention mechanism to learn the corresponding combination weights $\left(\boldsymbol a_t, \boldsymbol a_f, \boldsymbol a_p \right) \in \mathbb{R}^{N \times 3}$ as the attention values of $N$ nodes with the distributions $\boldsymbol{\mathcal D}_t, \boldsymbol{\mathcal D}_f, \boldsymbol{\mathcal D}_p$, respectively.
For node $v_i$, the corresponding topology structure
distribution is the $i$-th row of $\boldsymbol{\mathcal D}_t$.
We apply a nonlinear transformation to $\boldsymbol{\mathcal D}_t$, and then use one shared attention vector $\boldsymbol{q} \in \mathbb{R}^{1 \times d_q }$ to compute the attention value for node $v_i$ as the following:
\begin{equation}
\small
    \omega_t^i = \boldsymbol{q} \cdot \operatorname{tanh} \left( 
        \boldsymbol{W}_{a} \cdot ( \boldsymbol{\mathcal D}_{t,i} )^T + \boldsymbol{b}
    \right),
\end{equation}
where $\operatorname{tanh} (\cdot)$ is the activation function, $\boldsymbol{W}_a \in \mathbb{R}^{d_q \times d_{ \text{Meta}} }$ is the parameter matrix and $\boldsymbol b \in \mathbb{R}^{d_q \times 1}$ is the bias vector. Similarly, the attention values for node feature and positional identity distributions $\boldsymbol{\mathcal D}_f, \boldsymbol{\mathcal D}_p$ are obtained by the same procedure. 
We then normalize the attention values with the softmax function to get the final weight 
$a_{t}^{i} = \frac{ \exp \left( \omega_{t}^{i} \right) } { \exp \left( \omega_{t}^{i} \right) + \exp \left( \omega_{f}^{i} \right) + \exp \left(  \omega_{p}^{i} \right) }$. 
Large $a_{t}^{i}$ implies the topological structure distribution dominates the NLD. 
Similarly, we compute $a_{f}^{i}$ and $a_{p}^{i}$. 
Now for all the $N$ nodes, we have the learned weights $\boldsymbol{a}_{t}=\left[a_{t}^{i}\right],
\boldsymbol{a}_{f}=\left[a_{f}^{i}\right],
\boldsymbol{a}_{p}=$ $\left[a_{p}^{i}\right] \in \mathbb{R}^{N \times 1}$, 
and denote $\boldsymbol{a}_{t}=\operatorname{diag}\left(\boldsymbol{a}_{t}\right), \boldsymbol{a}_{f}=\operatorname{diag}\left(\boldsymbol{a}_{\boldsymbol{f}}\right)$ and $\boldsymbol{a}_{p}=\operatorname{diag}\left(\boldsymbol{a}_{p}\right)$. 
Then we integrate these three distributions to obtain the representation of NLD $\boldsymbol{\mathcal W}$ :
\begin{equation}
\small
\boldsymbol{\mathcal W}=
\boldsymbol{a}_{f} \odot \boldsymbol{\mathcal D}_{f}+\boldsymbol{a}_{t} \odot \boldsymbol{\mathcal D}_{t}+\boldsymbol{a}_{p} \odot \boldsymbol{\mathcal D}_{p}.
\end{equation}

\subsection{Adaptive Convolution}
\label{sec:model-b}
In the second stage, we elaborate the concrete algorithm based on the generated meta-weight $\boldsymbol{ \mathcal W }$. 
To adapt the graph convolution according to the information contained within the NLD, we propose the Decoupled Aggregation Weights and Independent Convolution Channels for node feature and topological structure.
On the one hand, we decouple the neighbor embedding aggregation weights based on $\boldsymbol{ \mathcal W }$ into $\boldsymbol{S}_f$ and $\boldsymbol{S}_t$ and balance them with a hyper-parameter $\alpha$.
The design ensures that the aggregation  take the most correlated information into account.
On the other hand, two additional Independent Convolution Channels for original topology and feature are introduced to boost the node representations.


\subsubsection{Decouple Topology and Feature in Aggregation}
\label{suction:model-b}
Recalling the discussion at the beginning of \autoref{sec:model-a}, using the same convolution over all the nodes and the pre-defined neighborhood, most existing GNNs can only project a fixed combination of node features and topological structure from neighborhood to node labels. Therefore, these GNNs achieve satisfactory results on graphs with simple NLD.
However, when the local distribution varies, the common $\operatorname{MEAN}(\cdot)$ or normalized aggregation can not recognize the difference and loss the distinguishment among nodes. 
Therefore,  we propose decoupling topology and feature in aggregation to adaptively weigh the correlation between neighbor nodes and ego nodes from the local distribution concept.
The following is the details of our decouple mechanism:
\begin{gather}
\small
\boldsymbol{H}^{(l+1)} = \sigma \left( \hat{ \boldsymbol{P}} \boldsymbol{H}^{(l)}  \boldsymbol{W}^{(l)} \right), \hat{\boldsymbol{P}} = \boldsymbol A \odot \boldsymbol S  \\
\boldsymbol{S} = \alpha \cdot \boldsymbol{S}_f + (1-\alpha) \cdot \boldsymbol{S}_t  \\
\boldsymbol{S}_f = \Psi_f \left( \boldsymbol{ \mathcal W},  \boldsymbol{X} \right) ,
\boldsymbol{S}_t = \Psi_t \left( \boldsymbol{ \mathcal W},  \boldsymbol{X}_{\text{LDP} } \right),
\label{eq:decouple}
\end{gather}
where $\boldsymbol{H}^{(l)} \in \mathbb{R}^{N\times d_l}$ denotes the hidden state of the $l$-th layer,  $\boldsymbol{W}^{(l)} \in \mathbb{R}^{d_l \times d_{l+1}}$ denotes the parameters of the $l$-th layer, $\sigma$ is the activation function, $\boldsymbol{S}$ is the integrated weight of aggregation, and $\boldsymbol{S}_f, \boldsymbol{S}_t$ are decoupled weights generated by two MLPs $\Psi_f, \Psi_t $, respectively.  
$\alpha$ is a hyper-parameter to balance the $\Psi_f$ and $\Psi_t$. 
Equipped with the external $\boldsymbol{X}$ and $\boldsymbol{X}_{\text{LDP}}$, 
the decouple of topology and feature in aggregation empowers the graph convolution to distinguish the different dependence on the corresponding factors and adjust itself to achieve the best performance.

\subsubsection{Independent Convolution Channels for Topology and Feature} 
GNNs learn node representations via alternately performing the aggregation and the transformation. 
In addition to the integrated information, the original node patterns in the feature and topological structure are essential in graph data mining. 
GNNs lose their advantages when the representations are over-smoothed \cite{deeper_insight} because the useful original node patterns are smoothed.
Recently, to alleviate the problem, some research work \cite{appnp,dagnn,sgc,GPRGNN} proposes that separating the procedure of aggregation and transformation. 
APPNP \cite{appnp} first generates predictions for each node based on its own features and then propagates them via a personalized PageRank scheme to generate the final predictions. 
GPR-GNN \cite{GPRGNN} first extracts hidden states for each node and then uses Generalized PageRank to propagate them.
However, the topological information is still entangled with the features even after separating the projection and propagation. Therefore, we propose two additional Independent Convolution Channels for the topology and feature information so that the model can maintain the original signals.
The detailed computation is:
\begin{equation}
\small
    \begin{aligned}
    \boldsymbol{H}^{(l+1)} = &\, \sigma  \Big( \left( 1-\lambda_1 - \lambda_2 \right) \hat{\boldsymbol{P}} \boldsymbol{H}^{(l)} + \lambda_1 \boldsymbol{H}_f^{(0)} + \\ & \lambda_2 \boldsymbol{H}_t^{(0)}   \Big) \cdot 
    \left( (1-\beta)\cdot \boldsymbol{I}_n + \beta \cdot \boldsymbol{W}^{(l)} \right),  \\ 
    \end{aligned}
    \label{eq:indenpent-channels}
\end{equation}
where 
$\boldsymbol{I}_n \in \mathbb{R}^{d_l \times d_{l+1}}$ with with ones as diagonal elements and  $\lambda_1, \lambda_2, \beta$ are hyper-parameters. $\boldsymbol{H}_f^{(0)}$ is the representation of initial node features and $\boldsymbol{H}_t^{(0)}$ is the representation of initial node topology embedding. 

Usually, in the previous GNNs like \cite{resgcn}, the $\boldsymbol{H}_f^{(0)}$ is the original node feature matrix $\boldsymbol X$. 
In our implementation, we apply fully-connected neural networks on the original node feature matrix $\boldsymbol X$ and the adjacency matrix $\boldsymbol{A}$ respectively to obtain the lower-dimensional initial representations for $\boldsymbol{H}_f^{(0)}$ and $\boldsymbol{H}_t^{(0)}$, so that when the feature dimension $F$ and the $N$ are large, we can guarantee an efficient calculation.


\subsection{Complexity} 
Here we analyse the time complexity of training MWGNN.
Let $L_d = \max \left( |\mathcal N _{1,k} |, \cdots, |\mathcal N _{N,k} |  \right)$ denote the length of node degree sequence length for structure pattern.
Because $d_\text{LDP}=5$, we leave out this term. 
In \autoref{sec:model-a}, the GRU costs $O\left(L_d(d_\text{Meta}^2 + d_\text{Meta} ) \right)$.
The MLP generating position embeddings costs  $O(N^2 d_\text{Meta})$ and we can reduce it to $O(|\mathcal E| d_\text{Meta} )$ with an efficient GPU-based implementation using sparse-dense matrix multiplications. 
Next, the integration of three distributions costs $ O\left( N d_\text{Meta} d_q^2  \right)$.
In the implementation, we set all the dimension of all hidden states as $d_{\text{hidden}}$.
The computation of \autoref{eq:decouple} costs $ O\left( N (F + d_\text{Meta} \right)$ as $F > \text{LDP}$.
The computation for $\boldsymbol{H}_f^{(0)}$ and $\boldsymbol{H}_t^{(0)}$ costs $O  \left( d_\text{hidden} (F N+ |\mathcal{E}| \right) $.  
The overall time  complexity of MWGNN is 
\begin{equation*}
\small
    O\left(L_d d_\text{hidden}^2 + N d_\text{hidden}^3+ NF d_\text{hidden} + |\mathcal E| (F+ d_\text{hidden} ) \right),
\end{equation*}
which matches the time complexity of other GNNs.

\section{Deep Analysis of MWGNN}
\label{sec:analysis}

\subsection{How Node Local Distribution Influence the Expressive Power of GNNs}

The discussion and empirical results above illustrate the importance of modeling NLD. 
Recalling \autoref{sec:model-a}, local edge homophily can be a relatively plain measurement for NLD. 
Therefore, without loss of generality, we take the local edge homophily and GCN \cite{kipf-gcn} as instances. 
We set $P$ as the random variable for local edge homophily with its distribution as $\mathcal{D}_P$. 
Thus, the variance of $P$ exhibits how the NLD differs throughout the graph. The larger the variance of $P$ is, the more complex the Local Distribution Pattern will be, and vice versa. 
To prevent trivial discussion, we make a few assumptions\footnote{The detailed assumptions can be found in \autoref{sec:theory}} to simplify our analysis without loss of generality. 
We derive a learning guarantee considering the variance of $P$ as follows.

\begin{theorem}
Consider $\mathcal{G} = \{\mathcal{V}, \mathcal{E}, \{\mathcal{F}_c, c \in  \{0,1\}\}, \{\mathcal{D}_P,
P\sim \mathcal{D}_P\},k\}$, which follows assumptions in Appendix A.  For any node $v_i \in  V,$ the expectation of its pre-activation output of 1-layer GCN model is as follows:
\begin{equation*}
\small
    \mathbb{E}\left[\mathbf{h}_i\right] = \boldsymbol W\left(\frac{1}{k + 1}\mu\left(\mathcal{F}_{y_i}\right)+\frac{k}{k + 1}\mathbb{E}_{P\sim\mathcal{D}_P,c\sim B\left(y_i,p\right),\boldsymbol x_j\sim\mathcal{F}_c}[\boldsymbol x_j]\right).
\end{equation*}
For any $t > 0$, the probability that the Euclidean distance between the observation $h_i$ and its expectation is larger than $t$ is bounded as follows:
\begin{equation*}
        \small
		\begin{aligned}
			&\mathbb{P}\left(\left\|\mathbf{h}_i-\mathbb{E}[\mathbf{h}_i]\right\|_2\geq  t\right) \\
			\leq&2d\exp \left(-\frac{((k+1)t_2/\rho(\boldsymbol W)+\sqrt{d}C_{\boldsymbol{x}}+\sqrt{d}C_{\mu})^2}{2k d\sigma^2 + 4\sqrt{d}C_{\boldsymbol  x}((k+1)t_2/\rho(\boldsymbol{W})+\sqrt{d}C_{\boldsymbol{x}}+\sqrt{d}C_{\mu})/3} \right),
		\end{aligned}
\end{equation*}
where $\sigma^2 = 4k C_{\mu}^2\operatorname{Var}\left[P\right]+k C_{\tau}$.
\label{theorem:theo1}
\end{theorem}

From \autoref{theorem:theo1} we demonstrate that the Euclidean distance between the output embedding of a node and its expectation is small when the variance of $P$ is relatively small. However, as the complexity of LDP increases, the upper bound of the learning guarantee will rapidly grow, which indicates that the traditional learning algorithm is no longer promising under this circumstance. 
Therefore, it is necessary to design an adaptive convolution mechanism to adjust the convolution operator based on nodes' various distribution patterns.

\subsection{Connection to existing GNNs} 
\paragraph{MWGNN on identical NLD degenerates into GNNs with three channels} 
When learning graphs with identical NLD, the $\boldsymbol{ \mathcal W}$ can not help to distinguish the nodes from the three distributions. However, we can still learn the adaptive weights by $\Psi_f$ and $\Psi_t$ with $\boldsymbol{X}, \boldsymbol{X}_\text{LDP}$. Moreover, if we remove the independent convolution channels, MWGNN degenerates to advanced GAT with two types of attention mechanisms.


\paragraph{Innovation design of MWGNN}
The two stages of MWGNN could be related to two types of methods. First, the form of Distribution-based Meta-Weight is like a kind of attention for aggregation. Unlike GAT measuring the similarity between nodes embeddings along the edges, we consider the local node distributions to generate the weights $\boldsymbol S$ from two aspects. The meta-weights give a description of a sub-graph around central nodes, and the pair-wise correlation is implicitly contained in the $\boldsymbol S$.
In addition, the design of Independent Convolution Channels is related to the residual layer. \cite{resgcn, gcnii} also introduced feature residual layers into GNNs. In addition to the residual layer of features, we also add a residual layer of topology. In this way, the final output of MWGNN contains three channels of representations, and the topology information is both explicitly and implicitly embedded into the representations.
\begin{table*}[htb]
    \caption{ 
    The summary of mean and standard deviation of accuracy over all runs.
    The best results for each dataset is highlighted in gray.
    "-" stands for Out-Of-Memory.
    }
    \label{table:real-benchmark}
    \centering
        \begin{tabular}{lcccccccc}
        \toprule 
        &  \textbf{Cora} & \textbf{Citeseer} & \textbf{Pubmed} & \textbf{Chameleon} & \textbf{Squirrel} & \textbf{Texas} & \textbf{Cornell}  \\ 
        \midrule
	
		MLP &  60.02 $\pm$ 0.75 &  53.36 $\pm$ 1.40 &  63.40 $\pm$ 5.03 &  48.50 $\pm$ 2.49 &  35.38 $\pm$ 1.66 &  75.95 $\pm$ 5.06 &  77.13 $\pm$ 5.32  \\ 
		GCN &  80.50 $\pm$ 0.50 &  70.80 $\pm$ 0.50 &  79.00 $\pm$ 0.30 &  38.22 $\pm$ 2.67 &  27.12 $\pm$ 1.45 &  58.05 $\pm$ 4.81 &  56.87 $\pm$ 5.29 \\ 
		GAT &  83.00 $\pm$ 0.70 &  72.50 $\pm$ 0.70 &  79.00 $\pm$ 0.30 &  43.07 $\pm$ 2.31 &  31.70 $\pm$ 1.85 &  57.38 $\pm$ 4.95 &  54.95 $\pm$ 5.63  \\ 
		\midrule
GPR-GNN &  \hl{83.69 $\pm$ 0.47} &  71.51 $\pm$ 0.29 &  79.77 $\pm$ 0.27 &  49.56 $\pm$ 1.71 &  37.21 $\pm$ 1.15 &  80.81 $\pm$ 2.55 &  78.38 $\pm$ 4.01   \\ 
		CPGNN-MLP-1 &  79.50 $\pm$ 0.38 &  71.76 $\pm$ 0.22 &  77.45 $\pm$ 0.24 &  49.25 $\pm$ 2.83 &  33.17 $\pm$ 1.87 &  80.00 $\pm$ 4.22 &  \hl{80.13 $\pm$ 6.47}   \\ 
		CPGNN-MLP-2 &  78.21 $\pm$ 0.93 &  71.99 $\pm$ 0.39 &  78.26 $\pm$ 0.33 &  51.24 $\pm$ 2.43 &  28.86 $\pm$ 1.78 &  79.86 $\pm$ 4.64 &  79.05 $\pm$ 7.78 \\ 
		CPGNN-Cheby-1 &  81.13 $\pm$ 0.21 &  69.72 $\pm$ 0.59 &  77.79 $\pm$ 1.06 &  48.29 $\pm$ 2.02 &  36.17 $\pm$ 2.87 &  76.89 $\pm$ 4.95 &  75.00 $\pm$ 7.64  \\ 
		CPGNN-Cheby-2 &  77.68 $\pm$ 1.55 &  69.92 $\pm$ 0.46 &  78.81 $\pm$ 0.28 &  50.95 $\pm$ 2.46 &  31.29 $\pm$ 1.26 & 76.89 $\pm$ 5.83 &  75.27 $\pm$ 7.80  \\ 
		AM-GCN &  81.70 $\pm$ 0.71 &  71.72 $\pm$ 0.55 &  -  &  56.70 $\pm$ 3.44 &  - &  74.41 $\pm$ 4.50 &  74.11 $\pm$ 5.53 \\ 
		H2GCN &  81.85 $\pm$ 0.38 &  70.64 $\pm$ 0.65 &  79.78 $\pm$ 0.43 &  59.39 $\pm$ 1.58 & 37.90 $\pm$ 2.02 &  75.13 $\pm$ 4.95 &  78.38 $\pm$ 6.62 \\ 
		\midrule
MWGNN &  83.30 $\pm$ 0.62 &  \hl{72.90 $\pm$ 0.47} &  \hl{82.30 $\pm$ 0.64} &  \hl{79.54 $\pm$ 1.28} &  \hl{75.41 $\pm$ 1.83} &  \hl{81.37 $\pm$ 4.27} &  79.24 $\pm$ 5.23 \\ 
        \bottomrule
    \end{tabular}
\end{table*}

\section{Experiment}
\label{sec:experiment}

\subsection{Datasets}
The proposed MWGNN is evaluated on nine real-world datasets 
and two types of synthetic datasets.


\subsubsection{Real-world datasets} 
The detailed information is in \autoref{table:real-world-datasets}. We use datasets considering both homophily and heterophily.
Cora, Citeseer, and Pubmed \cite{Planetoid} are widely adopted citation datasets with strong edge homophily; In contrast, Texas and Cornell \cite{Geom-GCN} are heterophily datasets; The situation of Chameleon and Squirrel \cite{WikipediaNetwork} are rather complex, with both homophily and heterophily combined.


\subsubsection{Synthetic datasets}

For synthetic benchmarks, we randomly generate graphs as below, referring to \cite{Abu-El-HaijaPKA19} and \cite{AMGCN}: 
(1) Labels: we randomly assign $C$ classes of labels to $N$  
nodes. 
(2) Node features: for the nodes with the same label, we use one Gaussian distribution to generate $d$-dimension node features. The Gaussian distributions for the $C$ classes of nodes have the same co-variance matrix, 
but the mean values of these $C$ Gaussian distributions are distinguishable. 
(3) Edges: the probability of building edges follows Bernoulli distributions controlled by  $\boldsymbol{B} \in \mathbb{R}^{C \times C}$. 
In particular, the probability of building an edge between node $v_i$ and $v_j$ follows the $\operatorname{Bernoulli} (\boldsymbol{B}_{y_i y_j} )$, where $y_i$ and $y_j$ are the node labels of $v_i$ and $v_j$. 

In addition, we further generate graphs combining different distributions (i.e. various Local Edge Homophily distributions) to demonstrate the situation where both homophilic and heterophilic data are mixed and tangled. Below are the details: 
(1) Generate two graphs $G_1(\mathcal V_1, \mathcal E_1)$ and $G_2(\mathcal V_2, \mathcal E_2)$ controlled by $\boldsymbol{B}^1$ and  $\boldsymbol{B}^2$ respectively. In details, we set the value of  $\boldsymbol{B}^1_{ij}$ high when $y_i = y_j$ and low when $y_i \not = y_j$ to build a graph with high homophily. 
Likewise, we set the value of  $\boldsymbol{B}^2_{ij}$ low when $y_i = y_j$ and high when $y_i \not = y_j$ to build a graph with low homophily. 
(2) Combine $G_1$ and $G_2$ by  randomly assign edges between nodes in $G_1$ and $G_2$ with a probability of $p$.

We generate three combined datasets: C.Homo
(short for Combined graphs with Homophily), C.Mixed (short for Combined graphs with mixed Homophily and Heterophily), and C.Heter. Detailed information of synthetic datasets is in \autoref{table:synthetic-datasets}.

\subsection{Settings}

We evaluate MWGNN on the semi-supervised node classification task compared with state-of-the-art methods. For citation datasets (Cora, Citeseer, and Pubmed), we use the public split recommended by \cite{Planetoid}, fixed 20 nodes per class for training, 500 nodes for validation, and 1000 nodes for the test. 
For web page networks (Texas, and Cornell), we adopt the public splits by \cite{Geom-GCN}, with an average train/val/test split ratio of 48\%/32\%/20\%\footnote{ \cite{Geom-GCN} claims that the ratios are 60\%/20\%/20\%, which is different from the actual data splits shared on GitHub.}. For Wikipedia networks (Chameleon and Squirrel) we use the public splits provided by \cite{WikipediaNetwork},  with an average train/val/test split ratio of 48\%/32\%/20\%. 

We use the Adam Stochastic Gradient Descent optimizer \cite{Adam_sgd} with a learning rate $\eta \in \{10^{-2}, 10^{-3}, 10^{-4}\}$, a weight decay of $5\times 10^{-4}$, and a maximum of 200 epochs with early stopping to train all the models. The number of hidden layers is set to 2, and the dimensions of hidden representations are set to 128 for fairness. For GAT-based models, the number of heads is set to 4.

\subsection{Evaluation on Real-world Benchmarks}

We compare the performance of MWGNN to the state-of-the-art methods\footnote{Note that on Chameleon and Squirrel, we reuse the results H2GCN reports, as they use public splits by \cite{Geom-GCN}. The results of GPR-GNN and CPGNN are different from their reports because they use their own splits rather than the public splits.} in \autoref{table:real-benchmark}.  Compared with all baselines, the proposed MWGNN generally achieves or matches the best performance on all datasets. 
Especially, MWGNN achieves an improvement of over 20\% on Chameleon and Squirrel, demonstrating the effectiveness of MWGNN while the graph data is not homophily- or heterophily- dominated but a combined situation of the both.


\subsection{Evaluation on Synthetic Benchmarks}

\begin{figure*} 
    \centering
    \begin{subfigure}[b]{0.55\linewidth}
        \includegraphics[]{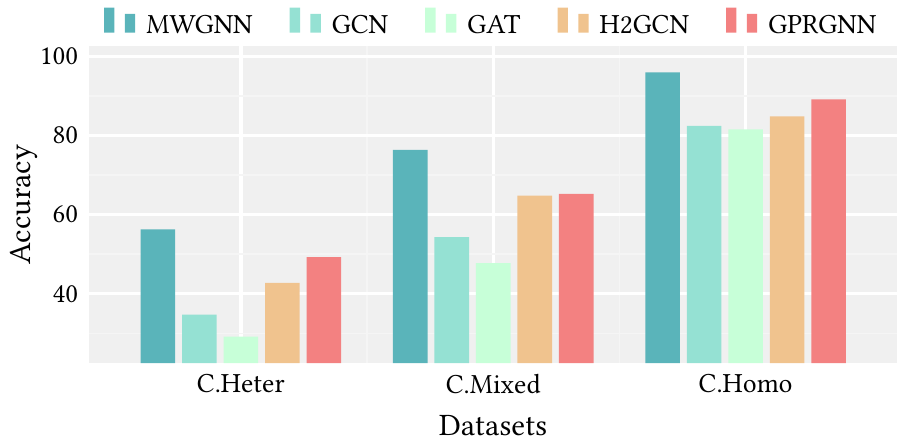}
    \end{subfigure} %
    \begin{subfigure}[b]{0.43\linewidth}    
        \includegraphics[]{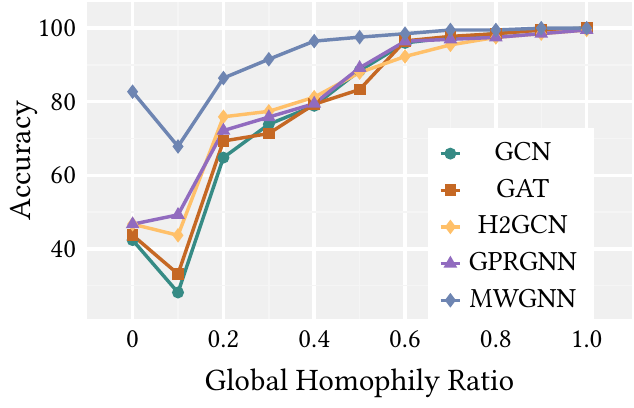}  
    \end{subfigure} 
    \caption{
        MWGNN and other baselines on synthetic datasets.
        }
    \label{fig:synthetic-benchmark}
\end{figure*}

To better investigate the performance of MWGNN on datasets with different Global/Local Edge Homophily distributions, we conduct experiments on a series of synthetic datasets. On the one hand, we test MWGNN on a set of synthetic graphs whose global edge homophily $h$ evenly ranges from 0 to 1 in \autoref{fig:synthetic-benchmark}.

MWGNN outperforms GCN and GAT on all homophily settings. On the other hand, we test our model on three combined graphs C.Homo, C.Mixed, C.Heter. As \autoref{fig:synthetic-benchmark} reveals, MWGNN disentangles the combined data well and achieves good results on all three synthetic graphs. Besides, on C.Mixed, GCN and GAT perform well on the sub-graph where Local Edge Homophily is high, with an accuracy over 95\%. On the contrary, up to 31\% nodes in the sub-graph where Local Edge Homophily is low are classified correctly. Meanwhile, MWGNN classifies the heterophily and homophily parts' nodes much better, with accuracy of 99\% and 69\%, separately. This observation suggests the significance of modeling the local distribution of graphs.


\subsection{Ablation Study}
To estimate the effectiveness of each part in MWGNN, we conduct an ablation study by removing one component at a time on our synthetic datasets, C.Heter, C.Mixed, and C.Homo. The results of the ablation study are in \autoref{table:ablation}.

\paragraph{Meta-Weight Generator} We remove the Distribution Based Meta-Weight Generator of MWGNN by removing the $\boldsymbol{\hat D}$ in $\Psi_f$ and $\Psi_t$. From the results, we can see that, in C.Heter and C.Mixed, removing the Distribution Based Meta-Weight Generator explicitly hinders the expressive power of MWGNN. C.Mixed is an example of graphs with different local distributions, homophily and heterophily parts are entangled. Besides, although C.Heter is generated by combining two heterophily graphs, the $\boldsymbol{B}^1$ and $\boldsymbol{B}^2$ used to construct them are different. Therefore, the distributions in C.Heter are not the same. Without the Distribution Based Meta-Weight Generator, MWGNN can no longer capture the distributions and generate convolution adaptively. In addition, the performance of the model only drops a little because the patterns in C.Homo are relatively the same. The removal of the Distribution Based Meta-Weight Generator has little impact on the performance of the model. These results support our claim that Distribution Based Meta-Weight Generator is capable of capturing different patterns, which could be used in the Adaptive Convolution. 

\begin{table}[htb]
    \caption{Ablation Study: Accuracy of MWGNN and its variants on three synthetic combined graph.}
    \label{table:ablation}
    \centering
    \begin{tabular}{lllllllll}
        \toprule
        &  & \textbf{C.Heter} & \textbf{C.Mixed} & \textbf{C.Homo} &   \\ 
        \midrule 
        & MWGNN  & 56.28  & 76.38  & 95.98\\
        & w/o $\boldsymbol D$  &  47.24 &  69.84 & 94.73\\
        & w/o $\boldsymbol D_f$  &  51.75 &  74.38 & 96.48\\
        & w/o $\boldsymbol D_t$  &  50.76 &  71.32 & 94.98\\
        & w/o $\boldsymbol D_p$  & 52.26 & 73.78  & 95.21\\
        & w/o Indep. Channels  & 53.77  & 73.87  & 86.73\\
        \bottomrule
    \end{tabular}
    
\end{table}

We further analyse the impact of the three components of the Distribution Based Meta-Weight Generator by removing the three local distributions $\boldsymbol{\mathcal D}_t, \boldsymbol{\mathcal D}_f$, and $\boldsymbol{\mathcal D}_p$ one at a time. On C.Mixed, the removal of the local distributions causes the performance of MWGNN to decline in varying extents respectively, which testifies the effect of the proposed $\operatorname{Att}(\cdot)$ in \autoref{sec:att}.

\textit{Independent Channels.} To demonstrate the performance improved through the Independent Convolution Channels for Topology and Feature, we test MWGNN after disabling it in \autoref{eq:indenpent-channels}.
The results suggest that our explicit utilization of the independent channel for topology information 
helps the model to achieve better results. Especially, the shared patterns in C.Homo lead to the a improvement.

\subsection{Parameter Analysis}

We investigate the sensitivity of hyper-parameters of MWGNN on Cora and Squirrel datasets. 
To be specific, the hyper-parameter $\alpha$ controlling the ratio of convolution weight $\boldsymbol S_f$ and $\boldsymbol S_d$ and the number of hops we consider when modeling the local distribution.

\textit{Analysis of $k$.}
We test MWGNN with different hop numbers $k$, varying it from 0 to 4. As \autoref{fig:parameter} shows, as $k$ increases, the performance is generally stable. Besides, a small $k$ is fair enough for MWGNN to reach satisfactory results.

\begin{figure} 
    \centering
    \begin{subfigure}[b]{0.48\linewidth}
        \includegraphics[]{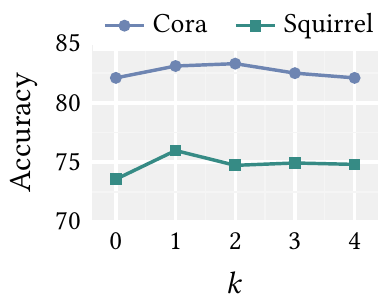}
    \end{subfigure} %
    \begin{subfigure}[b]{0.48\linewidth}    
        \includegraphics[]{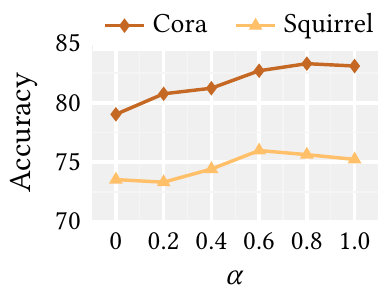}  
    \end{subfigure} 
    \caption{
        Parameter analysis over Cora and Squirrel on hop number $k$ and combine alpha $\alpha$.
        }
    \label{fig:parameter}
\end{figure}  

\textit{Analysis of $\alpha$.} To analyse the impact of $\alpha$ in 
\autoref{eq:decouple}, we study the performance of MWGNN with $\alpha$ ranging evenly from 0 to 1. 
Cora reaches its maximum at $\alpha=0.8$, while squirrel reaches its maximum at $\alpha=0.6$, which indicates that different graphs vary in the dependency on feature and topology. In addition, MWGNN is relatively stable when $\alpha$ changes around the maximum point.

\section{Related Work}

\subsection{Graph Neural Networks}
The Graph Neural Networks (GNNs) aim to map the graph and the nodes (sometimes edges) to a low-dimension space. Scarselli et al. \cite{GNNScarselli09} first propose a GNN model utilizing recursive neural networks to update topologically connected nodes' information recursively. Then, Bruna et al. expand GNNs to spectral space \cite{bruna-gcn}. Defferrard, Bresson, and Vandergheynst \cite{defferrard2016convolutional} generalize a simpler model, ChebyNet. 
Kipf and Welling \cite{kipf-gcn} propose Graph Convolution Networks (GCNs) which further simplify the graph convolution operation. 
GAT \cite{gat} introduces an attention mechanism in feature aggregation to refine convolution on graphs. GraphSAGE \cite{sage} presents a general inductive framework, which samples and aggregates feature from local neighborhoods of nodes with various pooling methods such as mean, max, and LSTM.

\subsection{GNNs on Heterophily Graphs}

Many models\cite{Geom-GCN,H2GCN,CPGNN,GPRGNN,AMGCN} design aggregation and transformation functions elaborately to obtain better compatibility for heterophily graph data and retain their efficiency on homophily data. Geom-GCN \cite{Geom-GCN} formulates graph convolution by geometric relationships in the resulting latent space. H2GCN  \cite{H2GCN} uses techniques of embedding separation, higher-order neighborhoods aggregation, and intermediate representation combination. CPGNN \cite{CPGNN} incorporates a label compatibility matrix for modeling the heterophily or homophily level in the graph to go beyond the assumption of strong homophily. GPR-GNN \cite{GPRGNN} adopts a generalized PageRank method to optimize node feature and topological information extraction, regardless of the extent to which the node labels are homophilic or heterophilic. AM-GCN \cite{AMGCN} fuses node features and topology information better and extracts the correlated information from both node features and topology information substantially, which may help in heterophily graph data.



\section{Conclusion}
\label{sec:conclusion}

In this paper, we focus on improving the expressive power of GNNs on graphs with different LDPs. We first show empirically that different LDPs do exist in some real-world datasets. 
With theoretical analysis, we show that
this variety of LDPs has an impact on the performance of traditional GNNs. To tackle this problem, we propose Meta-Weight Graph Neural Network, consisting of two key stages. First, 
to model the NLD of each node, we construct the Meta-Weight generator with multi-scale information, including structure, feature and position. Second, to decouple the correlation between node feature and topological structure, we conduct adaptive convolution with two aggregation weights and three channels. Accordingly, we can filter the most instructive information for each node and efficiently boost the node representations.
Overall, MWGNN outperforms corresponding GNNs on real-world benchmarks, while maintaining the attractive proprieties of GNNs. We hope MWGNN can shed light on the influence of distribution patterns on graphs, and inspire further development in the field of graph learning.

\begin{acks}
We would like to thank Zijie Fu and Yuyang Shi for their help on language polishing.
This work was supported by the National Natural Science Foundation of China (Grant No.61876006). 
\end{acks}

\bibliographystyle{ACM-Reference-Format}
\bibliography{citation}

\newpage

\appendix
\section{Detailed Information for Datasets}

\begin{table}[htbp]
    \caption{Statistics for real-world datasets.}
    \label{table:real-world-datasets}
    \centering
    \begin{tabular}{llllll}
        \toprule
        Dataset & $|\mathcal V|$ & $|\mathcal E|$ & $|\mathcal Y|$ & F & $ h $ \\
        \midrule
        \textbf{Cora} & 2708 & 10556 & 7 & 1433 & 0.81 \\
        \textbf{Citeseer} & 3327 & 9104 & 6 & 3703 & 0.74 \\
        \textbf{Pubmed} & 19717 & 88648 & 3 & 500 & 0.8 \\
        \textbf{Texas} & 183 & 325 & 5 & 1703 & 0.11 \\
        \textbf{Cornell} & 183 & 298 & 5 & 1703 & 0.31\\
        \textbf{Chameleon} & 2277 & 36101 & 5 & 1703 & 0.2 \\
        \textbf{Squirrel} & 5201 & 217073 & 5 & 2089 & 0.22 \\
        \bottomrule
    \end{tabular}
\end{table}

\begin{table}[htb]
    \caption{Statistics for combined synthetic datasets. $h_1, h_2$ denotes $h$ for separate graphs that are used to combine.}
    \label{table:synthetic-datasets}
    \centering
    \begin{tabular}{llllllll}
        \toprule
        Dataset & $|\mathcal V|$ & $|\mathcal E|$ & $|\mathcal Y|$ & F & $ h $& $h_1$& $h_2$ \\
        \midrule
        \textbf{C.Homo} & 1000 & 22937 & 5 & 100 & 0.51 & 0.50 & 0.50 \\
        \textbf{C.Mixed} & 1000 & 22705 & 5 & 100 & 0.39 & 0.10 & 0.75 \\
        \textbf{C.Heter} & 1000 & 15201 & 5 & 100 & 0.20 & 0.20 & 0.20\\
        \bottomrule
    \end{tabular}
\end{table}


\section{Proofs of THEOREM 4.1}\label{sec:theory}
\label{appendix:theory}

\paragraph{Notation} $\mathcal{G}=(\mathcal{V,E})$ denotes a graph with $d$-dimension node feature vector $\boldsymbol x_i$ for $v_i\in \mathcal{V}$. Features of all dimensions are bounded by a positive scalar $C_{\boldsymbol  x}$. $ y_i$ denotes the label for node $v_i$. $P$ is a random variable for local edge homophily with its distribution as $\mathcal{D}_P$. $\mathbf{h}_i$ denotes the embedding of node $ v_i$. $\boldsymbol W \in R^{d \times d}$ denotes the parameter matrix of 1-layer GCN model. $\rho(\boldsymbol W)$ denotes the largest singular value of $\boldsymbol W$. 
\paragraph{Assumptions on Graphs}
(1) $\mathcal{G}$ is k-regular. It can prevent us from trivial discussion on the expression of GCN and help us focus on the mechanism of massage passing.
(2) The features of node $v_i$ are sampled from feature distribution $\mathcal{F}_{ y_i}$, i.e, $\boldsymbol x_i\sim \mathcal{F}_{ y_i}$, with $\mu(\mathcal{F}_{y_i})=\mathbb E \left[\boldsymbol x_i|y_i\right]$ and $\tau(\mathcal{F}_{y_i})=\mathbb E \left[\boldsymbol x_i\circ\boldsymbol x_i|y_i\right]$. Similarly, $\mathcal{F}_{\overline{y}_i}$ denotes the feature distribution of nodes having labels other than $y_i$.
(3) Dimensions of $\boldsymbol x_i$ are independent to each other and they are all bounded by a positive scalar $C_{\boldsymbol x}$.
(4) Dimensions of $\mu(\mathcal{F}_{y_i})$ and $\tau(\mathcal{F}_{y_i})$ are all bounded by positive scalars $C_{\mu}$ and $C_{\tau}$, respectively.
(5) For node $v_i$, its local edge homophily $p$ is sampled from pattern distribution $\mathcal{D}_P$.  If $P = p$, node $v_i$'s neighbors' labels are independently sampled from Bernoulli distribution $B\left(y_i,p\right)$.
\begin{lemma}[Bernstein's inequality]
Let $X_1,...,X_n$ be independent bounded random variables with $ X_i\in  [a, b]$ for any $i$, where $-\infty < a \leq  b < +\infty$. Denote that  $\overline{X}=\frac{1}{n}(X_1+...+X_n)$ and $\sigma^2 = \sum_{i=1}^nVar[X_i]$. Then for any $t  > 0$, the following inequalities hold:
\begin{equation*}
    \begin{aligned}
        \mathbb{P}\left(\overline{X}-E\left[\overline{X}\right]\geq t\right)\leq \exp\left(-\frac{nt^2}{2\sigma^2+2t(b-a)/3}\right),\\
        \mathbb{P}\left(\overline{X}-E\left[\overline{X}\right]\leq-t\right)\leq \exp\left(-\frac{nt^2}{2\sigma^2+2t(b-a)/3}\right).\\
    \end{aligned}
\end{equation*}
\label{lemma:bern}
\end{lemma}

\begin{lemma}[Hoeffding’s Inequality]
Let $X_1, . . . , X_n$ be independent bounded random variables with $X_i\in  [a, b]$ for any $i$, where $-\infty < a \leq  b < +\infty$.  Denote that $\overline{X}=\frac{1}{n}(X_1+...+X_n)$. Then for any $t > 0$,the following
inequalities hold:
\begin{equation*}
    \begin{aligned}
        \mathbb{P}\left(\overline{X}-E\left[\overline{X}\right]\geq t\right)\leq \exp\left(-\frac{2nt^2}{(b-a)^2}\right),\\
        \mathbb{P}\left(\overline{X}-E\left[\overline{X}\right]\leq-t\right)\leq \exp\left(-\frac{2nt^2}{(b-a)^2}\right).\\
    \end{aligned}
\end{equation*}
\label{lemma:hoeff}
\end{lemma}

\begin{lemma}[The Union Bound]
For any events $A_1,A_2,...,A_n$, we have
\begin{equation*}
    \mathbb{P}\left(\bigcup_{i=1}^nA_i\right) \leq \sum_{i=1}^n\mathbb{P}\left(A_i\right).
\end{equation*}
\label{lemma:ub}
\end{lemma}

\begin{theorem}
Consider $\mathcal{G} = \{\mathcal{V}, \mathcal{E}, \{\mathcal{F}_c, c \in  \{0,1\}\}, \{\mathcal{D}_P,
P\sim \mathcal{D}_P\},k\}$, which follows assumptions (1) - (5). For any node $v_i \in  V,$ the expectation of its pre-activation output of 1-layer GCN model is as follows:
\begin{equation*}
    \mathbb{E}\left[\mathbf{h}_i\right] = \boldsymbol W\left(\frac{1}{k + 1}\mu\left(\mathcal{F}_{y_i}\right)+\frac{k}{k + 1}\mathbb{E}_{P\sim\mathcal{D}_P,c\sim B\left(y_i,p\right),\boldsymbol x_j\sim\mathcal{F}_c}[\boldsymbol x_j]\right).
\end{equation*}
For any $t > 0$, the probability that the Euclidean distance between the observation $\mathbf{h}_i$ and its expectation is larger than $t$ is bounded as follows:
	\begin{equation*}
		\begin{aligned}
			&\mathbb{P}\left(\left\|\mathbf{h}_i-\mathbb{E}[\mathbf{h}_i]\right\|_2\geq  t\right) \\
			\leq&2d\exp \left(-\frac{((k+1)t_2/\rho(\boldsymbol{W})+\sqrt{d}C_{\boldsymbol{x}}+\sqrt{d}C_{\mu})^2}{2kd\sigma^2 + 4\sqrt{d}C_{\boldsymbol  x}((k+1)t_2/\rho(\boldsymbol{W})+\sqrt{d}C_{\boldsymbol{x}}+\sqrt{d}C_{\mu})/3} \right),
		\end{aligned}
	\end{equation*}
	where $\sigma^2 = 4kC_{\mu}^2\operatorname{Var}\left[P\right]+kC_{\tau}$.
\label{theorem:theo}
\end{theorem}



\begin{proof}
For a single layer GCN model, the process can be written in the following form for node $v_i$
\begin{equation*}
\mathbf{h}_i=\sum_{j\in \left\{i\right\}\cup \mathcal{N}\left(i\right)}\frac{1}{k + 1}\boldsymbol W\boldsymbol x_j,
\end{equation*}
so that the expectation of $\mathbf{h}_i$ can be derived as follows:
\begin{equation*}
    \begin{aligned}
        \mathbb{E}\left[\mathbf{h}_i\right]&=\mathbb{E}\left[\sum_{j\in \left\{i\right\}\cup \mathcal{N}\left(i\right)}\frac{1}{k + 1}\boldsymbol W \boldsymbol x_j\right]\\ 
        &=\frac{1}{k + 1}\boldsymbol W\mathbb{E}\left[\boldsymbol x_i\right]+\sum_{j\in \mathcal{N}\left(i\right)}\frac{1}{k + 1}\boldsymbol W\mathbb{E}\left[\boldsymbol x_j\right]\\ 
        &= \frac{1}{k + 1}\boldsymbol W\mu\left(\mathcal{F}_{y_i}\right)+\frac{k}{k + 1}\boldsymbol W\mathbb{E}_{P\sim\mathcal{D}_P,c\sim B\left(y_i,p\right),\boldsymbol x_j\sim\mathcal{F}_c}\left[\boldsymbol x_j\right]\\ 
        &= \boldsymbol W\left(\frac{1}{k + 1}\mu\left(\mathcal{F}_{y_i}\right)+\frac{k}{k + 1}\mathbb{E}_{P\sim\mathcal{D}_P,c\sim B\left(y_i,p\right),\boldsymbol x_j \sim\mathcal{F}_c}\left[\boldsymbol x_j\right]\right).\\ 
    \end{aligned}
\end{equation*}
When $P$ is given, the conditional distribution of $\boldsymbol x_j$ follows
	\[\left(\boldsymbol{x}_j|P=p\right)\sim\begin{cases}
		p\mathcal{F}_{y_i}+(1-p)\mathcal{F}_{\overline{y}_i},&j\in\mathcal{N}(i),\\
		\mathcal{F}_{y_i},&j=i.
	\end{cases}.\]
	Let $\boldsymbol x_j^l,l = 1,...,d$ denote the $l$-th element of $\boldsymbol x_j$. Then, for any dimension $l$,  $\left\{\boldsymbol x_j^l,j\in \mathcal{N}\left(i\right)\right\}$ is a set of independent random variables. When $j\in\mathcal{N}(i)$
	\begin{equation*}
		\begin{aligned}
			\operatorname{Var}\left[\mathbb{E}\left[\boldsymbol{x}_j^l\big |P\right]\right]=&\operatorname{Var}\left[P\mu^l(\mathcal{F}_{y_i})+(1-P)\mu^l(\mathcal{F}_{\overline y_i})\right]\\
			=&(\mu^l(\mathcal{F}_{y_i})-\mu^l(\mathcal{F}_{\overline y_i}))^2\operatorname{Var}\left[P\right].\\
			\leq&4C_{\mu}^2\operatorname{Var}\left[P\right].\\
			\mathbb{E}\left[\operatorname{Var}\left[\boldsymbol{x}_j^l\big 
			|P\right]\right]=&\mathbb E\left[\mathbb E\left[\left(\boldsymbol{x}_j^l\right)^2\big |P\right]-\mathbb E\left[\boldsymbol{x}_j^l\big |P\right]^2\right]\\
			=&\mathbb E\left[P\tau^l(\mathcal{F}_{y_i})+(1-P)\tau^l(\mathcal{F}_{\overline y_i})\right]\\
	    &-\mathbb E\left[\left(P\mu^l(\mathcal{F}_{y_i})+(1-P)\mu^l(\mathcal{F}_{\overline y_i})\right)^2\right]\\
			=&-\mathbb E\left[P^2\right](\mu^l(\mathcal{F}_{y_i})-\mu^l(\mathcal{F}_{\overline y_i}))^2\\
			+\mathbb E\left[P\right]&\left(-2(\mu^l(\mathcal{F}_{y_i})-\mu^l(\mathcal{F}_{\overline y_i}))\mu^l(\mathcal{F}_{\overline y_i})+\tau^l(\mathcal{F}_{y_i})-\tau^l(\mathcal{F}_{\overline y_i})\right)\\
			+&\tau^l(\mathcal{F}_{\overline y_i})\\
			\leq&C_{\tau}.
		\end{aligned}
	\end{equation*}
	By the law of total variance, 
	\begin{equation*}
		\operatorname{Var}\left[\boldsymbol{x}_j^l\right]=\operatorname{Var}\left[\mathbb E \left[\boldsymbol{x}_j^l\big |P\right]\right]+\mathbb E\left[\operatorname{Var}\left[\boldsymbol{x}_j^l\big |P\right]\right]\leq4C_{\mu}^2\operatorname{Var}\left[P\right]+C_{\tau}.
	\end{equation*}
	Then for any $t_1 > 0$, we have the following bound by applying lemma \autoref{lemma:bern} and lemma \autoref{lemma:ub}:
	\begin{equation*}
		\begin{aligned}
			&\mathbb P\left[\frac{1}{k + 1}\left\|\sum_{j\in\{i\}\cup \mathcal{N}\left(i\right)}\left(\boldsymbol x_j^l-\mathbb E\left[\boldsymbol x_j^l\right]\right)\right\|\geq t_1\right]\\
			\leq&\mathbb P\left[\frac{1}{k}\left\|\sum_{j\in\mathcal{N}\left(i\right)}\left(\boldsymbol x_j^l-\mathbb E\left[\boldsymbol x_j^l\right]\right)\right\|\geq \frac{(k+1)t_1+C_{\boldsymbol{x}}+C_{\mu}}{k}\right]\\
			&+\mathbb P\left[\left\|\boldsymbol x_i^l-\mathbb E\left[\boldsymbol x_i^l\right]\right\|\geq C_{\boldsymbol{x}}+C_{\mu}\right]\\
			\leq&2 \exp \left(-\frac{((k+1)t_1+C_{\boldsymbol{x}}+C_{\mu})^2}{2k\sigma^2 + 4C_{\boldsymbol  x}((k+1)t_1+C_{\boldsymbol{x}}+C_{\mu})/3} \right).
		\end{aligned}
		\label{eq:bern}
	\end{equation*}
	where 
	$\sigma^2=4kC_{\mu}^2\operatorname{Var}\left[P\right]+kC_{\tau}$. By applying lemma \autoref{lemma:ub} to \autoref{lemma:bern}, the following holds:
	\begin{equation*}
		\begin{aligned}
			&\mathbb{P} \left( \frac{1}{k + 1}\left\|\sum_{j\in \left\{i\right\}\cup\mathcal{N}\left(i\right)}\left(\boldsymbol x_j-\mathbb{E}\left[\boldsymbol x_j\right]\right)\right\|_2\geq t_2 \right)\\
			\leq&\sum_{l=1}^{d}\mathbb P\left[\frac{1}{k + 1}\left\|\sum_{j\in\{i\}\cup \mathcal{N}\left(i\right)}\left(\boldsymbol x_j^l-\mathbb E\left[\boldsymbol x_j^l\right]\right)\right\|\geq\frac{t_2}{\sqrt{d}}\right]\\
			\leq&2d\exp \left(-\frac{((k+1)t_2+\sqrt{d}C_{\boldsymbol{x}}+\sqrt{d}C_{\mu})^2}{2kd\sigma^2 + 4\sqrt{d}C_{\boldsymbol  x}((k+1)t_2+\sqrt{d}C_{\boldsymbol{x}}+\sqrt{d}C_{\mu})/3} \right).
		\end{aligned}
		\label{eq:bound}
	\end{equation*}
	Furthermore, we have
	\begin{equation*}
		\begin{aligned}
			\left\|\mathbf{h}_i-\mathbb{E}\left[\mathbf{h}_i\right]\right\|_2 & \leq  \frac{1}{k + 1}\left\|\boldsymbol W\left(\sum_{j\in\left\{i\right\}\cup \mathcal{N}\left(i\right)}\boldsymbol x_j-\mathbb{E}\left[\boldsymbol x_j\right]\right)\right\|_2\\ 
			&\leq \frac{\left\|\boldsymbol W\right\|_2}{k+1}\left\|\sum_{j\in\left\{i\right\}\cup \mathcal{N}\left(i\right)}\boldsymbol x_j-\mathbb{E}\left[\boldsymbol x_j\right]\right\|_2\\ 
			&\leq \frac{\rho(\boldsymbol W)}{k+1}\left\|\sum_{j\in\left\{i\right\}\cup \mathcal{N}\left(i\right)}\boldsymbol x_j-\mathbb{E}\left[\boldsymbol x_j\right]\right\|_2,\\ 
		\end{aligned}
	\end{equation*}
	where $\|\boldsymbol W\|_2$ denotes the matrix 2-norm of $\boldsymbol W$ and $\rho(\boldsymbol W)$ denotes the largest singular value of $\boldsymbol W$. Then, for any $t > 0$, we have
	\begin{equation*}
		\begin{aligned} 
			&\mathbb{P}\left(\left\|\mathbf{h}_i-\mathbb{E}\left[\mathbf{h}_i\right]\right\|_2  \geq  t\right) leq \mathbb{P}\left(\frac{\rho(\boldsymbol W)}{k+1}\left\|\sum_{j\in \left\{i\right\}\cup \mathcal{N}\left(i\right)}\boldsymbol x_j-\mathbb{E}\left[\boldsymbol x_j\right]\right\|_2\geq  t\right)\\ 
			&=\mathbb{P}\left(\frac{1}{k+1}\left\|\sum_{j\in \{i\}\cup \mathcal{N}(i)}\boldsymbol x_j-\mathbb{E}[\boldsymbol x_j]\right\|_2\geq  \frac{t}{\rho(\boldsymbol W)}\right)\\
			&\leq 2d\exp \left(-\frac{((k+1)t_2/\rho(\boldsymbol{W})+\sqrt{d}C_{\boldsymbol{x}}+\sqrt{d}C_{\mu})^2}{2kd\sigma^2 + 4\sqrt{d}C_{\boldsymbol  x}((k+1)t_2/\rho(\boldsymbol{W})+\sqrt{d}C_{\boldsymbol {x}}+\sqrt{d}C_{\mu})/3} \right),\\
		\end{aligned}
	\end{equation*}
	which concludes the proof.
\end{proof}


\end{document}